\documentclass{article}

     \PassOptionsToPackage{numbers, compress}{natbib}

     \usepackage[final]{sty/neurips_2021}

\usepackage{hyperref}       
\usepackage[utf8]{inputenc} 
\usepackage[T1]{fontenc}    
\usepackage{url}            
\usepackage{booktabs}       
\usepackage{amsfonts}       
\usepackage{nicefrac}       
\usepackage{microtype}      
\usepackage{xcolor}         
\usepackage{xspace}
\usepackage[]{algorithm2e}
\usepackage{amsfonts, amsmath, amssymb, amsthm}
\usepackage{thmtools,thm-restate} 
\usepackage{dsfont}

\newtheorem{lemma}{Lemma}
\newtheorem{theorem}{Theorem}

\theoremstyle{definition}

\newtheorem{remark}{Remark}

\usepackage{sty/notations}  

\title{Bandits with many optimal arms}

\author{%
  Rianne de Heide \\ 
  INRIA Lille and CWI Amsterdam\\
  \texttt{r.de.heide@cwi.nl} \\
   \And
   James Cheshire \\
   Otto von Guericke University Magdeburg \\
   \texttt{james.cheshire@ovgu.de} \\
   \And
   Pierre M\'enard \\
   Otto von Guericke University Magdeburg \\
   \texttt{pierre.menard@ovgu.de} \\
   \And
   Alexandra Carpentier \\
   University of Potsdam \\
   \texttt{carpentier@uni-potsdam.de} \\
}

\begin{document}

\maketitle

\begin{abstract}
We consider a stochastic bandit problem with a possibly infinite number of arms. We write $\pstar$ for the proportion of optimal arms and $\Delta$ for the minimal mean-gap between optimal and sub-optimal arms. We characterize the optimal learning rates both in the cumulative regret setting, and in the best-arm identification setting in terms of the problem parameters $T$ (the budget), $\pstar$ and $\Delta$. For the objective of minimizing the cumulative regret, we provide a lower bound of order $\Omega(\log(T)/(\pstar\Delta))$ and a UCB-style algorithm with matching upper bound up to a factor of $\log(1/\Delta)$. Our algorithm needs $\pstar$ to calibrate its parameters, and we prove that this knowledge is necessary, since adapting to $\pstar$ in this setting is impossible. For best-arm identification we also provide a lower bound of order $\Omega(\exp(-cT\Delta^2\pstar))$ on the probability of outputting a sub-optimal arm where $c>0$ is an absolute constant. We also provide an elimination algorithm with an upper bound matching the lower bound up to a factor of order $\log(T)$ in the exponential, and that does not need $\pstar$ or $\Delta$ as parameter. Our results apply directly to the three related problems of competing against the $j$-th best arm, identifying an $\epsilon$ good arm, and finding an arm with mean larger than a quantile of a known order.
\end{abstract}


\section{Introduction}
\label{sec:intro}

In the classical stochastic multi-armed bandit model -- see \citep{lattimore2020bandit} for a recent survey -- a learner interacts with an environment in several rounds. At each round, the learner chooses an \emph{arm} to play, and receives a random reward from the associated probability distribution. Popular settings are respectively the fixed budget \emph{cumulative regret setting}~\citep{robbins1952some}, and \emph{best-arm identification setting}~\citep{even2002pac, bubeck2009pure, audibert2010best}. In the first setting, the learner is interested in maximizing the sum of rewards gathered -- or minimizing the cumulative regret -- and in the best-arm identification setting, the learner is asked at the end of the game to output a guess for the arm with the largest mean reward, and is interested in the quality of this guess -- typically measured by the probability of error in the guess. 

In most of the papers that concern this topic, it is assumed (i) that there is a single optimal arm, i.e.~arm with highest mean, and (ii) that the number of arms is bounded and small when compared to the time horizon, i.e.~the number of rounds where the player is allowed to choose an arm. However in many realistic applications, it is not the case, for example in image classification, mining of resources, personalized medicine, or hyperparameter tuning (see \citep{berry1997bandit} for more examples). And while it is clear that in all generality, the task of the learner becomes unsolvable if the number of arms is too large, it intuitively makes sense that if the proportion of optimal arms is also large, this should help the learner.

In this paper, we lift both assumptions summarised in (i) and (ii) and study both the cumulative regret and best-arm identification setting. See Section~\ref{sec:intro-related-work} for literature related to this that we will discuss later. We will focus on the \textit{problem dependent setting} and will aim at characterising optimal learning rates depending on the proportion of optimal arms, and on the minimal gap between the mean of an optimal arm and the mean of a sub-optimal arm.

\subsection{Setting}

We consider a setting with a (potentially infinite) set of arms $\cA$, which we call the \emph{reservoir}. Each arm $a \in \cA$ is associated with a probability distribution $\nu_a$, which we assume to be supported on $[0,1]$, and we denote its mean by $\mu_a$. Write $\mu^* = \max_{a\in\cA} \mu_a$ for the highest mean\footnote{We assume that it is attained for some arm(s).}, $\mu_{sub} = \sup_{a\in \cA: \mu_a \neq \mu^*} \mu_a$ for the second highest mean, and $\Delta = \mu^* - \mu_{sub}$ for the associated minimal gap. We will focus throughout this paper on the case where $\Delta >0$.

We further assume that there exists a partition 
$\cA = \cA^* \cup \cA_{sub}$ 
such that each arm $a\in \cA^*$ is optimal, i.e.~$\mu_a = \mu^*$, and each arm $a\in \cA_{sub}$ is sub-optimal, i.e.~$\mu_a \leq \mu_{sub}$. 
We assume that the agent can pick arms uniformly at random from the reservoir $\cA$\footnote{In case of infinite $\cA$, one can obviously not sample from a uniform distribution. Our analysis extends to general distributions on $\cA$.}, and this arm belongs either to the set $\cA^*$ with probability $\pstar$, i.e.~there is a proportion $\pstar$ of optimal arms in the reservoir; or it belongs to the set $\cA_{sub}$ with probability $1 - \pstar$, i.e.~there is a proportion $1-\pstar$ of sub-optimal arms in the reservoir.


The learner interacts with the environment in several rounds $t = 1, 2, \ldots, T$, where we fix the time horizon $T$. At each round $t \leq T$, the learner chooses an arm $a_t$ by either picking a new arm from the reservoir $\cA$ or playing a past arm, and gets a reward $Y_t \sim \nu_{a(t)}$. The arm choice depends only on the past observations, the past arm choices, and possibly some exogenous randomness. The rewards for each arm $a$ are i.i.d.\ random variables with mean $\mu_a$ unknown to the learner.

\paragraph{Cumulative regret setting.} The first setting we study is that of minimizing the \emph{cumulative regret}. This setting enforces the \emph{exploration-exploitation trade-off}: the learner needs to balance exploratory actions to get a better estimate of the reward distributions, and exploitative actions to maximize the total return -- and minimise the associated cumulative regret. The cumulative regret is the difference between the sum of expected rewards the learner would have obtained by only choosing the arm with the highest mean reward, and the sum of expected rewards she actually collected:
\[
R(T)  = \sum_{t=1}^T \mustar - \mu_{a(t)}\,.
\]

\paragraph{Best-arm identification setting} In the second setting we study, we are interested in identifying an arm with the highest mean reward. 
At the end of $T$ rounds, the agents selects an arm $\ha_T$ and aims at minimising the probability of outputting an arm with sub-optimal mean:
\[
\error(T) = \PP{\ha_T \notin \mathcal A^*}.
\]
A closely related popular measure of error is the~\emph{simple regret}, which is not discussed in this paper.

\paragraph{Equivalent settings} Firstly, our setting is directly applicable to the problem of competing against the $j$-th best arm, where we assume w.l.o.g.\ the arms to be ordered according to their means. Indeed our setting translates to this if we replace $\pstar$ by $j/K$ and $\Delta$ by the gap between the $j/2$-th and the $j+1$-th best arm, i.e.\ $\Delta = |\mu_{j/2} - \mu_{j+1}|$. Secondly, our setting is directly applicable to that of identifying an $\epsilon$ good arm, and thirdly, our setting is directly applicable to finding any arm in the reservoir with a mean larger than the quantile of a known order  -- see the discussion in Section~\ref{sec:intro-related-work}.

\subsection{Contributions}
We characterise the optimal learning rates both for the cumulative regret setting, and for best-arm identification, for our problem described above. We characterise the optimal learning rates in terms of the problem parameters $T, \pstar,$ and $\Delta$.

In order to describe our results, let us write for $\bar \Delta>0$, $\bar \pstar \in [0,1)$: $\mathfrak B_{\bar \Delta, \bar \pstar}$, for the set of bandit problems whose reservoir distribution is such that $\pstar \geq \bar \pstar$ and such that $|\bar \mu^*-  \mu_{sub}| \geq \bar \Delta$. 

\paragraph{Cumulative regret} We provide an algorithm, \emph{that takes $\pstar$ as a parameter}, that is such that (see Theorem~\ref{thm:UBcumulative}) 
$$\mathbb E R(T) \leq O\left(\frac{\log T \log(1/\Delta)}{\pstar \Delta}\right).$$
Conversely, we prove in Theorem~\ref{thm:LBcumulative2} that for $\bar \pstar \leq 1/4$ and $\bar \Delta \leq 1/4$, and for any algorithm, there exists a problem in $\mathfrak B_{\bar \Delta, \bar \pstar}$ such that 
$$\mathbb E R(T) \geq \Omega \left( \frac{\log T}{\bar \pstar \bar\Delta}\right).$$
These two bounds match up to a multiplicative factor of order $\log(1/\Delta)$. They highlight the intuitive fact that we should pay the number of arms in the rate only relative to the number of optimal arms -- i.e.~only through $\pstar$. Indeed, the probability of picking an optimal arm in the reservoir when sampling uniformly at random being $\pstar$, if we sample about $1/\pstar$ arms at random from the reservoir, we will have sampled one optimal arm with constant probability -- so that $1/\pstar$ plays the same role as the number of arms.

Having said that, there is a main conceptual difficulty in order to get a rate that is tight in terms of its dependence in $T$. If we sample only $1/\pstar$ arms from the reservoir, the probability of having no optimal arms in the chosen set of arms is also a constant -- so that the regret is linear in $T$. It is therefore essential to sample \emph{more} arms. In order to have a logarithmic regret in $T$, we need to sample at least about $\log T/\pstar$ arms from the reservoir -- in which case at least one of them is optimal with probability polynomially decaying with $T$. But if we do this, we get a regret of order $\frac{(\log T)^2}{\pstar \Delta}$, as there are about $\log T/\pstar$ sub-optimal arms whenever $\pstar$ is not too close to $1$. This is much larger than the bound that we have, where the dependence on $T$ is only $\log T$. In order to achieve this bound, we need to take into account the fact that when sampling $\log T/\pstar$ arms from the reservoir, there is typically not just $1$, but $\log T$ optimal arms with high probability -- and leverage this fact both in our algorithm and in the associated proof. We describe this in more detail in Section~\ref{sec:cumulative_regre_ub}. 


\paragraph{Best-arm identification} We provide an algorithm \emph{that does not take $\pstar$ as a parameter}, such that,
$$\error(T) \leq O\left( \log(T)\exp\left(-c  \frac{T\Delta^2\pstar}{ \log(T)}\right) \right),$$
where $c$ is some universal constant. 
Conversely, we prove that for $\pstar \leq 1/4$ and $\Delta \leq 1/4$, and for any algorithm, there exists a problem in $\mathfrak B_{\bar \Delta, \bar \pstar}$ such that $\error(T) \geq \Omega \left( \exp\left(- c T \Delta^2 \pstar\right) \right)$,
where $c>0$ is some universal constant. These two bounds match in order up to a factor of order $\log(T)$ in the exponential, it is an open question here whether this term is necessary or not.

These bounds highlight the intuitive fact that we should pay the number of arms in the rate only relative to the number of optimal arms -- i.e.~only through $\pstar$. As in the cumulative regret setting, if we sample about $1/\pstar$ arms at random from the reservoir, we will have sampled one optimal arm with constant probability -- so that $1/\pstar$ plays the same role as the number of arms.

As in the cumulative regret setting, there is again a main conceptual difficulty in order to get a rate that is tight in terms of its dependence in $T$. If we sample only $1/\pstar$ arms from the reservoir, the probability of having no optimal arms in the chosen arms is also a constant -- which is way smaller than the targeted best-arm identification probability. In order to have at least one optimal arm in the set of arms picked from the reservoir with a probability that decays exponentially with $\pstar T\Delta^2$, the number of arms that have to be sampled should be larger than $T\Delta^2$. But if we do this, 
we get an upper bound on the probability of error that is of constant order -- which is much larger than the bound that we have. In order to obtain our upper bound, we need to take into account the fact that when sampling $T\Delta^2$ arms from the reservoir, there is typically not just $1$, but $\pstar T\Delta^2$ optimal arms with high probability -- and leverage this fact both in our algorithm and in the associated proof. We describe this in more detail in Section~\ref{sec:simple_regret_ub}.

\paragraph{Adaptation to $\pstar$: diverging pictures for cumulative regret and best-arm identification} The algorithm for cumulative regret takes (a lower bound on) $\pstar$ as parameter, but the algorithm for best-arm identification does not take anything related to $\pstar$ or $\Delta$ as a parameter. And so, while our algorithm for best-arm identification is adaptive to $\pstar$ and $\Delta$, our cumulative regret algorithm is adaptive to $\Delta$ but not $\pstar$. In Section~\ref{sec:cumulativeregret_impossibilityOfAdaptation} we prove that it is not just a weakness of our analysis, but that it is \emph{impossible to adapt to $\pstar$ when it comes to the cumulative regret}. The phenomenon of adaptation to the problem hyper-parameters being possible for best-arm identification but not for cumulative regret, was observed earlier: In the $\mathcal X$-armed bandit setting  \cite{locatelli2018adaptivity} show it is impossible to adapt to smoothness and \cite{hadiji2019polynomial} further classifies the cost of adaptation in this case. \cite{zhu2020regret} explore the cost of adaptation to $\pstar$ for the problem independent case where the number of arms is large.

\subsection{Related work}\label{sec:intro-related-work}

\paragraph{Finite and small number of arms.}
The regret-minimization setting, introduced by \citep{robbins1952some}, has been well-studied for \emph{finite}-armed bandit models. Algorithms for this problem fall into several categories: algorithms based on upper-confidence bounds (UCB) for the unknown arm means \citep{katehakis1995sequential,auer2002finite, auer2010ucb, cappe2013kullback}, algorithms that exploit a posterior distribution on the means, such as Thompson Sampling \citep{thompson1933likelihood, kaufmann2012thompson}, and many more such as explore-then-commit \citep{garivier2016explore}  and phased-elimination \citep{even2006action}. Logarithmic instance-dependent lower bounds have already been obtained in the seminal paper by \citep{lai1985asymptotically}, and were generalized later, e.g.\ by \citep{burnetas1996optimal}, see \citep{garivier2019explore} for an overview and simple proofs. In the setting where the number of arms $|\cA|$ is finite and not too large -- much smaller than $T$ -- a classical problem dependent upper bound on the expected cumulative regret is\footnote{\label{note1} In the case where $\cA$ is finite otherwise the quantity below is infinite.}\addtocounter{footnote}{-1}\addtocounter{Hfootnote}{-1} 
\begin{equation}\label{eq:cumfinite}
    \sum_{a \in \cA \setminus \cA^*} \left(\frac{8\log T}{\mu^* - \mu_a}+2 \right)\leq |\cA_{sub}| \frac{\log T}{\Delta}+2 |\cA_{sub}|.
\end{equation}
The bound in the RHS is tight if all sub-optimal arms have the same gap $\Delta$. Moreover, this regret bound asymptotically matches the lower bound by \citep{burnetas1996optimal} up to a multiplicative constant. 
In the case where there are infinitely many sub-optimal arms, on the other hand, this upper bound is infinite, \textit{even when the proportion of optimal arms $\pstar$ is large and where one would hope for better performances}. 


The fixed-budget best-arm identification setting was introduced by \citep{bubeck2009pure, audibert2010best} and has been widely studied. It is well-known that algorithms that are optimal for cumulative-regret minimization cannot yield optimal performance for best-arm identification \citep{bubeck2011pure, kaufmann2017learning}. Write\footnotemark $H = \sum_{a\in \cA\setminus \cA^*} \frac{1}{(\mu^* - \mu_a)^2} \leq \frac{|\cA_{sub}|}{\Delta^2}.$
The bound in the RHS is tight if all sub-optimal arms have gap $\Delta$. It is proven by \citep{audibert2010best} that given $H$, there exists an algorithm such that the probability of misidentifying an optimal arm is of order $\exp\left(-cT/H\right)$,
where $c>0$ is some universal constant. 
In the case where there is \textit{a single optimal arm} this bound is provably optimal \citep{carpentier2016tight} when $H$ is known.
However, in the case where there are infinitely many sub-optimal arms this upper bound is larger than $1$ and thus vacuous, \textit{even when the proportion of optimal arms $\pstar$ is large and where one would hope for better performances}. 

Importantly, our results in both settings extend to finite bandits. Furthermore we do not need infinite $\cA$ for our results to be near optimal. In the finite setting with $K$ arms and $\pstar K$ optimal arms the problem is strictly harder than one with $\frac{1}{\pstar}$ arms and a single optimal arm. Indeed, the latter problem would correspond to one where the learner receives, as additional information, a partition of the set of $K$ arms in $\frac{1}{\pstar}$ groups, where one of the groups contains all optimal arms, and the others are only composed of sub-optimal arms. One can then see that we match the classical UB and LB for the finite bandit problem, up to $\log(1/\Delta)$ terms.



\paragraph{Large to infinite number of arms.}

The setting with an infinite number of arms -- and sometimes also many optimal arms -- has been studied in different settings.

A setting that is very related to ours is the infinitely many-armed setting where a distribution is assumed on the reservoir -- called the reservoir distribution. At each round, the learner can pull a previously queried arm, or a new arm that is sampled according to the reservoir distribution. A classical assumption on the reservoir is that the proportion of $\bar \Delta$-near optimal arms is of larger order than $\bar \Delta^{-\alpha}$ for any $\bar \Delta$. 
This setting been studied for both cumulative regret minimization \citep{berry1997bandit, wang2008algorithms, bonald2013two, david2014infinitely} and for best-arm identification \citep{carpentier2015simple, aziz2018pure, chaudhuri2017pac}. A classical strategy is to select a subset of arms from the reservoir, large enough so that it contains a near optimal arm with high probability, and to use classical bandit strategies on these arms. The minimax order of magnitude of the cumulative regret is then $\sqrt{T} \lor T^{\alpha/(\alpha+1)}$ and for the simple regret it is $T^{-1/2} \lor T^{-1/\alpha}$. 

Related results have also be obtained in the setting where the number of arms is finite, but large -- i.e.~$K > T$ -- and under related assumptions on the frequency of near-optimal arms~\citep{zhu2020regret}. 
While our setting is extremely related to this setting, the assumption about the frequency of near-optimal arms differs in the above literature from the assumption we make in this paper. Their bounds are not dependent upon $\Delta$ -- they assume $\forall k \in [K], \mu_k \in [0,1]$, and instead focus on achieving semi adaptivity in regards to an unknown $\alpha^*$, where $\alpha^* := \inf\{ \alpha : K/|S_{*}| < T^{\alpha} \}$. In the context of our setting $T^\alpha$ would act as a upper bound on $1/\pstar$. They propose an algorithm with user defined parameter $\beta$ that has no guarantees on regret for $\beta < \alpha$. And while our assumption is more restrictive, we also expect to obtain much smaller optimal rates. Our results differ from this stream of literature in the same way that, in the classical MAB, \textit{problem dependent results differ from problem independent results}. 


Another setting takes a regularity assumption on the reservoir distribution around $\mu^*$ -- that is, the proportion of arms in the reservoir whose gap is of order greater than $\bar \Delta$ is bounded above by a function of $\bar \Delta$, typically $\bar \Delta^{\alpha}$, where $\alpha$ is the regularity coefficient. For best-arm identification adaptivity is possible without knowledge of $\alpha$ and \cite{carpentier2015simple} provide algorithms for the simple regret with LB matching up to $\log(T)$ terms. In the case of cumulative regret \cite{wang2008algorithms} and \cite{bonald2013two} again provide near optimal results but in the case of \textit{known} $\alpha$. While the above literature considers a weaker assumption on the reservoir distribution, their results are also considerably weaker than our own. For best-arm identification they identify a sub optimal arm whose distance to the optimal arm is bounded polynomially with $T$. For cumulative regret the regret is bounded polynomially with $T$. These bounds are in both cases much larger than our bounds -- which essentially reflects that their assumption are weaker.

Closer to our setting are the works \citep{chaudhuri2017pac} and \citep{aziz2018pure}, where they try to find any arm in the reservoir with a mean larger than the quantile of a known order (with respect to the reservoir distribution) with high probability. This can be seen as the fixed confidence version of our setting for best-arm identification where the order of the quantiles is our known proportion of optimal arms $\pstar$ and the gap $\Delta$ is the difference between the first and the second quantile of order $\pstar$.  Precisely, \citep{aziz2018pure} provide an algorithm that can find an arm above the quantile of order $\pstar$ with probability at least $1-\delta$ in less than $H_{\Delta,\pstar} \log(1/\delta)^2$ samples on average, where $H_{\Delta,\pstar} \approx 1/(\pstar \Delta^2)$ is the problem dependent constant. The fixed confidence result of~\citep{aziz2018pure} translates, in the fixed budget setting, into an upper bound on the probability of error $\error(T)$ of order $\exp(-c \sqrt{T\pstar \Delta^2})$ where $c>0$ is some universal constant -- which is much larger than our bound for large $T$. Similarly, \citep{chaudhuri2018quantile} consider the regret with respect to a fixed quantile of order $\pstar$ of the distribution of the means in the reservoir which is again quite related to the regret in our setting. They obtain an algorithm with a bound on cumulative regret of order $R(T) \leq O\big(1/\pstar + \sqrt{(T/\pstar) \log(\pstar T)}\big)$, for any $\Delta>0$ -- in this sense, this analysis is problem independent.

Also closely related is the paper \citep{samuels2020complexity} which deals with identifying an $\epsilon$ good arm -- in the case where there are many such $\epsilon$ good arms, with high probability. Again this can be seen as a fixed confidence version of our setting, with the proportion of $\epsilon$ good arms being equivalent to our $\pstar$. However, the focus of their results differs considerably to our own. Specifically, in our setting, Theorem 2 of \citep{samuels2020complexity} provides an upper bound on the expectation of a stopping time for epsilon good arm identification, of the order $\bar{\mathcal{H}}\log(\bar{\mathcal{H}})$ where $\bar{\mathcal{H}} \approx 1/(\pstar\Delta^2)\log(1/\delta)$ but this bound does not hold in high probability, which would be necessary if one wished to directly compare their results to ours. Indeed for the stopping time of their algorithm to be bounded in high probability one would need to pay a $\log(1/\delta)^2$ term, corresponding to $\exp(-\sqrt{\Delta^2 p^* T})$ in our setting, see Remark 4 in \citep{samuels2020complexity} and page 15 in the appendix of the full version \citep{katzsamuels2019true}. The focus of \citep{samuels2020complexity} is instead to get more complete gap dependent bounds, considering also the gaps within the epsilon good arms but as mentioned their results cannot be applied directly to our setting and, as they point out, extending their approach to include high probability guarantees would be strictly sub optimal compared to our results.

We can also view the \emph{most-biased coin problem} studied by \citep{chandrasekaran2014finding} and \citep{jamieson2016power} as a particular instance of our setting where all optimal arms are distributed according to a Bernoulli distribution $\Ber(\mustar)$ and any sub-optimal arm is distributed according to the \emph{same} Bernoulli distribution $\Ber(\mu^-)$. The goal is then to identify an optimal arm with high probability with as few samples as possible. Precisely, \citep{jamieson2016power} prove that they can find an optimal arm with probability at
least $1-\delta$ with $\log\!\big(1/(\pstar \Delta^2)\big) \frac{\log(1/\delta)}{\pstar \Delta^2} $ samples in expectation when $\mustar,\mu^-$ and $\pstar$ are unknown to the agent and with $\frac{\log(1/\delta)}{\pstar \Delta^2}$ samples if $\pstar$ is known. It is also worth mentioning the problem of $\pstar$ estimation for the biased coin problem. For unknown $\pstar$ and $\Delta$, \citep{Lee2021uncertainty} describe, in the fixed confidence setting, the optimal learning rate for estimating $\pstar$, up to an additive error $\epsilon$, of the order $\frac{\pstar}{\epsilon^2 \Delta^2}\log(1/\delta)$. 

The translation of the result from~\citep{jamieson2016power} to the fixed budget setting is much closer to our result, as it would provide a bound of order $\exp\left(-c T\pstar \Delta^2/ \log(1/(\pstar\Delta^2))\right)$ where $c>0$ is some universal constant. This is very similar to our bound, but there is a main difference: we do not assume that there are just two possible distribution for the arms as \citep{jamieson2016power} -- the set $\cA_{sub}$ of sub-optimal arms might contain arms of diverse means, all being at a gap more than $\Delta$ from $\mu^*$. This makes the problem \textit{significantly more difficult} -- in particular regarding the adaptation to $\pstar$ -- since in our setting, it is impossible to estimate the minimal gap $\Delta$, see Section~\ref{sec:disscussion}. In fact, extending to a more general reservoir is an open question of interest left at the end of the above paper.

Otherwise, there are some other formulations of the infinitely-many armed bandit problem that are quite popular, but very different from our setting, and that we mention here for completeness. Many works are devoted to the setting where there is some topological relation between the index of the arms, and the mean of the arms~\citep{kleinberg2008multi, bubeck2011x, grill2015black}. This setting is often referred to as the $\mathcal X-$armed bandit setting, and not related to our work as we do not make such topological assumptions. Finally, a paper in which the setting is close to ours, but where the goal is very different, is the one by \citep{juneja2019sample}. The authors consider a partition of the (infinite) space $\Omega$ of K-armed bandit models $\nu = (\nu_1, \ldots, \nu_K)$, and want to identify for a given bandit model $\mu \in \Omega$ the correct partition component it belongs to. 

\paragraph{Fixed confidence to fixed budget setting}
In the fixed confidence setting for best-arm identification, given some $\delta > 0$, one aims to bound the expected number of samples one needs to correctly identify an optimal arm with probability greater than $1-\delta$. With our best-arm identification upper bound (Theorem \ref{thm:UBsimple}) in mind, we can essentially translate our result to the fixed confidence setting by considering $\delta=\exp\left(-\frac{T\pstar\Delta^2}{\log(1/\Delta)}\right)$,
and solving for $T$. This leads to a upper bound on the number of samples \OurAlgorithmSimple needs to be $\delta$-approximately correct of: $\frac{\log\left(\frac{1}{\delta}\right)\log\left(\frac{1}{\Delta}\right)}{\pstar\Delta^2}$. The papers \cite{jamieson2016power} and \cite{aziz2018pure} both deal with settings very related to our own but from the fixed confidence perspective. \cite{aziz2018pure} deals with quantile estimation and as highlighted above their results can be applied to our setting but with a significantly worse bound on probability of error of order $\exp(\sqrt{T\pstar \Delta})$. In \cite{jamieson2016power} the problem of best-arm identification is tackled directly but with strong restriction on the reservoir distribution, they consider the case were all sub optimal arms are identically distributed. 
\section{Cumulative regret}
\label{sec:cumulative_regret}
We first present an algorithm and prove an upper bound on its cumulative regret, and then we present a problem-dependent lower bound that shows we match the regret bound up to poly-log terms in $\Delta$. Lastly, we provide a theorem to the effect that adaptation to the proportion of optimal arms $\pstar$ is not possible in this setting.

\subsection{Upper bound}
\label{sec:cumulative_regre_ub}
We present \OurAlgorithm for cumulative regret minimization. This algorithm is an Upper Confidence Bound (UCB) type algorithm \citep{lattimore2020bandit}. We first sample a set $\cL$ of arms large enough such that with high probability (of order $1-1/T$) there is a proportion 
of order $\pstar$ optimal arms. Then we build an upper confidence bound on the empirical mean of each sampled arm, see~\eqref{eq:def_ucb},
where $\hmu_{a}^t$ is the empirical mean of arm $a$ at time $t$ and $N_{a}^t$ the number of times arm $a$ was pulled until time $t$. At time $t$ we pull the arm $a\in\cL$ with the highest upper confidence bound $U_{a}^t$. The complete procedure is detailed in Algorithm~\ref{alg:OurAlgorithmRegret}.
Notably, we do not tune the upper confidence bounds such that they are exceeded with probability less than $1/T$, as for finite-armed bandits. In that setting, a common choice is to have bonuses of the form $\hmu_{a}^t + \sqrt{2\log(T)/N_{a}^t}$, see \citep{lattimore2020bandit}. Instead we use an exploration function that does not depend on $T$, such that the upper confidence bounds are exceeded with probability smaller than a fixed constant, see~\eqref{eq:def_ucb}.
Thus we only pay a constant regret of order $\log(1/\Delta)$ on the set of sampled arms $\cL$. This is made possible by leveraging the fact that we know that there is a proportion of order $\pstar$ optimal arms.
\begin{algorithm}[ht]
\caption{Sampling UCB}
\label{alg:OurAlgorithmRegret}
\SetKwInput{KwData}{Input}
    \SetAlgoLined
   {\bfseries Input:} $\gamma \in (0,1)$, $L\geq 1$
    
    {\bfseries Initialize:}  Pick $\cL$, with $|\cL|= L$, arms from the reservoir $\cA$. Sample each arm once.

    \For{$t= L+1 $ to $T$}{
    Compute for each arm $a\in\cL$ the quantity
        \begin{equation}
\label{eq:def_ucb}
      U_{a}^t = \hmu_{a}^t + \sqrt{\frac{\gamma^2(1-\gamma)^{-1}/4+\log(\pi^2/6) +2 \log(N_{a}^t)}{2N_{a}^t}},
\end{equation}
        
        Play
        $
        a_t = \argmax_{a \in \cL} U_{a}^t.
        $

    }
\end{algorithm}

We prove the following regret bound for \OurAlgorithm in Appendix~\ref{app:regret_proofs}.

\begin{theorem}\label{thm:UBcumulative}
For $T\geq 2 $, $\gamma \in (0,1)$ and $L = \big\lceil 4\log(T)/(\pstar\gamma^2) \big\rceil$, the expected cumulative regret of \OurAlgorithm is upper bounded as follows:
  \[
    \E R(T) \leq O\left( \frac{\log(T) \log(1/\Delta)}{\pstar \Delta}\right)\,,
  \]
 see the end of the proof for a precise bound, i.e.~\eqref{eq:final_ub_regret}.
\end{theorem}

Note that this bound matches the lower bound of Theorem~\ref{thm:LBcumulative2} of Section~\ref{sec:cumulative_regret_lb}, for $T$ large enough and up to a $\log(1/\Delta)$ multiplicative factor. Also, $L$ can be calibrated with a lower bound on $\pstar$ instead of $\pstar$, but this lower bound will appear in the rate instead of $\pstar$.

\begin{remark} \label{rem:arms}
Algorithm~\OurAlgorithm samples $L$ arms uniformly at random from the reservoir. What we mean by this is that each arm is pulled at random from $\cA$ \textit{independently from the other pulled arms}. In other words, by doing this, we potentially artificially create several independent copies of the same arm -- which might seem counter-intuitive, but is formally not a problem.\\
What this anyway implies is that the case $|\cA| \leq L$ is not a problem -- with this idea of independent copies, we can pull more arms from the reservoir than the number $|\cA|$ of arms.
\end{remark} 

 \begin{remark}
 Our algorithm is reminiscent of that of \cite{jamieson2014lil}, which, as our own, uses a UCB which does not depend on the time horizon, but only on the number of times an arm has been pulled. However, they do so for different reasons, namely to adapt to the infinite time horizon of the fixed confidence setting.
 \end{remark}

\subsection{Lower bound}
\label{sec:cumulative_regret_lb}
We can prove an equivalent of the \citep{lai1985asymptotically} lower bound for finite-armed bandits for our setting. The following theorem is proved in Appendix~\ref{app:regret_proofs}.
\begin{theorem}\label{thm:LBcumulative2}
Consider $\Delta \in (0,1/4)$ and $\pstar \in (0,1/4]$.
For any bandit algorithm, there exists a bandit problem in $\mathfrak B_{\Delta,\pstar}$ such that
$$\mathbb E R(T) \geq \min\left( \frac{1}{60} \frac{\max\left\lbrace\log(\Delta^2 T/16),0\right\rbrace}{\pstar \Delta}, \sqrt{T} \right)$$
\end{theorem}
Note that if we consider the gap $\Delta$ and the proportion of optimal arms $\pstar$ as fixed and $T$ large in comparison, i.e.~$\Delta\gg \sqrt{1/T} $, then our lower bound is of order $\log(T)/(\pstar\Delta)$. This is the problem-dependent regime that we consider in this paper. On the contrary, if $\Delta \approx \sqrt{1/T} $ then our lower bound is of order $\sqrt{T}$. This is rather the problem-independent regime studied by \citep{chaudhuri2018quantile}. We can make a parallel between the lower bound in our setting and the one for finite-armed bandits. Indeed, if we consider that the proxy for the number of arms is $|\cA| \sim 1/\pstar$ which implies that there is $\pstar |\cA| \sim 1$ optimal arm, then we recover the problem-dependent lower bound of order $|\cA| \log(T) / \Delta$, if there are $|\cA|-1$ sub-optimal arms with gap $\Delta$.

\subsection{Impossibility of adapting to $\pstar$}\label{sec:cumulativeregret_impossibilityOfAdaptation}
The following theorem shows that in the setting of minimizing the cumulative regret, it is impossible to adapt to the proportion of optimal arms $\pstar$. The theorem is proved in Appendix~\ref{app:regret_proofs}.

\begin{theorem}\label{thm:unkownpcumulative}
Let $\pstar \leq \frac{1}{4}$ and $c>0$ such that $T \geq 4\left(\frac{c\log(T)}{\pstar\Delta^2}\right)^2$. For any bandit algorithm $\mathfrak A$ such that for all bandit problems in $\mathfrak B_{\Delta,\pstar}$, we have,
\[\mathbb{E}R(T) \leq \frac{c\log(T)}{\pstar\Delta}\]
one has that $\forall \qstar \leq \frac{4\pstar}{c}$ there exists a problem in $\mathfrak B_{\Delta,\qstar}$ such that
\[\E R(T) \geq \frac{\sqrt{T}\Delta}{4} \;.\]
\end{theorem}

\begin{remark}
The \OurAlgorithm algorithm takes a user defined parameter $\gamma$ (which can be taken as a universal constant) and $L$, which should be calibrated depending on (a lower bound on) $\pstar$. While this is necessary, it is important to not that none of the parameters requires knowledge of $\Delta$.
\end{remark}


\section{Best-arm identification}
\label{sec:simple_regret}
We present our \OurAlgorithmSimple algorithm for best-arm identification, together with an upper bound on the probability of outputting a sub-optimal arm; next we prove a lower bound, which is matched by our upper bound up to a $1 / \log(T)$ factor in the exponential.

\subsection{Upper bound}
\label{sec:simple_regret_ub}
As its name suggests, the \OurAlgorithmSimple algorithm (summarized in Algorithm~\ref{alg:OurAlgorithmSimpleRegret}) works by successive elimination of arms -- through the update at round $i$ of a set $\cA_i$ -- although with a twist. We begin by sampling approximately $T$ arms at the first round. Namely, we first select a set $\cA_1$ of 
$|\cA_1| = \lfloor \bar c T/\log T\rfloor$ arms taken at random from the reservoir, for some constant $\bar c >0$. Then at each round we use a $T/\log T$ fraction of our budget to sample the arms in our set. And so at round $i$ we sample each arm in the set $\cA_i$ a number of 
$t_i= \lfloor  \bar c T/(|\cA_i| \log T)\rfloor$. We then eliminate half of the arms based on the arms' empirical means -- namely, we just keep the $\big\lfloor |\cA_i|/2\big\rfloor \lor 1$ arms in $\cA_i$ that have highest empirical means -- and introduce an additional number of arms sampled from the reservoir distribution -- namely $\big\lfloor |\cA_i|/4\big\rfloor$ -- such that the final size of our arm set is reduced by $\frac{3}{4}$. At the end of the budget, we have one arm remaining -- due to the choices of $\bar c$ --
which is the arm that we return.
Note that Remark~\ref{rem:arms} applies here too so that it is not a problem if $|\cA|$ is smaller than the number of arms required by the algorithm. Theorem~\ref{thm:UBsimple} is proved in Appendix~\ref{app:BAIproofs}.

\begin{algorithm}[ht]
\caption{\OurAlgorithmSimple}
\label{alg:OurAlgorithmSimpleRegret}
    \SetAlgoLined
   {\bfseries Input:} $\bar c$\\ 
   set $i\leftarrow 1$\\
    \While{$i<\log T/\bar c$}{

   Sample each arm in $\mathcal A_i$ a number  $t_i = \lfloor \bar c T/(|\cA_i| \log T )\rfloor$ of times and compute their empirical means $(\hat \mu_{i}(a))_{a \in \mathcal A_i}$
   \\
   Put in $\mathcal A_{i+1}$ the $1 \lor \lfloor |\mathcal A_{i}|/2\rfloor$ arms that have highest empirical means $(\hat \mu_{i}(a))_{a \in \mathcal A_i}$, and add on top of that $\lfloor |\mathcal A_i|/4\rfloor$ new arms taken at random from the reservoir \\
 $i \leftarrow i+1$ 
   }
 Return any $\ha_T$ in $\cA_{i}$ 
\end{algorithm}

\begin{theorem}\label{thm:UBsimple} Set $\bar c = \log(4/3)$ . \OurAlgorithmSimple satisfies  
\begin{align*}
    \mathbb P(\ha_T \in\cAstar) \geq 1 - 2\log(T) \exp\left(- c \frac{ \Delta^2 \pstar T}{ \log T} \right),
\end{align*}
    where $c = \bar c/ 19200$
\end{theorem}

\begin{remark}
\OurAlgorithmSimple works by discarding many sub-optimal arms and few optimal arms in each round, so that at the end, when just one arm remains, it is optimal with high probability. A key element is that \OurAlgorithmSimple adds \emph{fresh arms} from the reservoir at each round. This is to ensure that our algorithm is adaptive to $\pstar, \Delta$, as ensured by Theorem~\ref{thm:UBsimple}. Whenever the arms in $\cA_i$ are pulled less than about $\Delta^{-2}$ times, there is no guarantee on what happens when half of the arms are eliminated. Therefore, we have to make sure that when the algorithm arrives at a round $i$ such that $t_i \gtrsim \Delta^{-2}$, the proportion of optimal arms is of larger order than $\pstar$ with high enough probability. This is ensured by adding the fresh arms added from the reservoir.
Note that for some arm distributions, we do not need to add fresh arms and the algorithm would function also by just halving at each step the number of arms. Indeed, in the case where all arms follow a Bernoulli distribution, in terms of preserving the proportion of optimal arms, one can prove that halving the set of arms according to the empirical means is no worse than random halving of the set. Thus, in this case, with high probability we increase the proportion of optimal arms at each step, without diminishing it. This is however specific to the case of Bernoulli distributions and some other parametric families, and it is an open question whether this would be true in general.

\begin{remark}
The successive halving strategy our algorithm for best-arm identification is based on was first introduced by \citep{karnin2013almost}, however, without the trick of adding fresh arms, as they didn't need to be adaptive to $\pstar$.
\end{remark}

\end{remark} 

\subsection{Lower bound}
\label{sec:simple_regret_lb}

The following Theorem provides a lower bound on the probability of error for best arm identification in our setting. The proof of Theorem \ref{thm:LBsimple} can be found in Appendix \ref{app:BAIproofs}.
\begin{theorem}\label{thm:LBsimple}
Consider $\Delta \in (0,1/4)$ and $\pstar \in [0,1/4]$.
 For any bandit algorithm, there exists a bandit problem in $\mathfrak B_{\Delta,\pstar}$ such that
$$\error(T) \geq \frac{1}{4} \exp\left(-T\pstar \frac{\Delta^2}{32} \right).$$
\end{theorem}
In proving the above theorem we essentially show that an agent cannot accurately distinguish between two cases: $\mu^* = \frac{1}{2}$ and $\mu^* = \frac{1}{2} + \Delta$. That is, we consider two reservoirs $\Nu_0$ and $\Nu_1$ where $\mu^*_{0} = \frac{1}{2}$ and $\mu^*_{1} = \frac{1}{2} + \Delta$. Using a coupling argument we bound the $\KL$ divergence between the distribution of samples collected on $\Nu_0$ and $\Nu_1$. The results then follows by application of Bretagnolle-Huber's inequality.

\section{Experiments}
\label{sec:experiments}
We conduct a preliminary set of experiments to test the performance of our algorithms. Specifically, for cumulative regret we compare our \OurAlgorithm to the QRM1 algorithm by \citep{chaudhuri2018quantile} and the SR algorithm by \citet{zhu2020regret}. For simple regret we compare our \OurAlgorithmSimple to the BUCB algorithm by \citep{samuels2020complexity}. In both cases our performance appears comparable to the literature. See Appendix~\ref{app:experiments} for details. 

\section{Conclusion and open questions}
\label{sec:disscussion}

Classifying optimal learning rates on the continuous armed bandit problems with a proportion of optimal arms and general reservoir distribution has been a question of interest in the literature for some time, see \cite{jamieson2016power}. Recent papers -- \cite{aziz2018pure} and \cite{zhu2020regret}, while focused on a slightly different setting, have considerably weaker results when applied to our setting. Therefore, we believe our results mark a significant improvement in the state of the art. An extension of our results would be to remove the $\log(1/\Delta)$ discrepancy between UB and LB for cumulative regret. However, this appears non-trivial and in particular we struggle to see how a UCB based strategy would achieve this tighter bound in the case of the cumulative regret. Another possibility for further work is an expansion of our setting. Consider the arm reservoir $\cA$ partitioned into $K$ possible distributions, each with associated probability $p_k$. Let $k^* = \argmax_{[K]} \mu_k$ and take gaps $(\Delta_k)_{[K]} = \left(\mu_{k^*} - \mu_k\right)_{[K]}$. One could then consider more detailed bounds, dependent on the sequence $\left((p_k,\Delta_k)\right)_{[K]}$  as opposed to just $\pstar$ and the smallest gap. The main difficulty here would be to deal with the case where some $p_k$ are much smaller than the proportion $\pstar$ corresponding to the optimal arm.



\paragraph{Acknowledgements} 
The work of J. Cheshire is supported by the Deutsche Forschungsgemeinschaft (DFG) GRK 2297 MathCoRe. 
The work of P. M\'enard is supported by the SFI Sachsen-Anhalt for the project RE-BCI. 
The work of A. Carpentier is partially supported by the Deutsche Forschungsgemeinschaft (DFG) Emmy Noether grant MuSyAD (CA 1488/1-1), by the DFG - 314838170, GRK 2297 MathCoRe, by the FG DFG, by the DFG CRC 1294 'Data Assimilation', Project A03, by the Forschungsgruppe FOR 5381 „Mathematische Statistik im Informationszeitalter – Statistische Effizienz und rechentechnische Durchführbarkeit“, Project 02, by the Agence Nationale de la Recherche (ANR) and the DFG on the French-German PRCI ANR ASCAI CA 1488/4-1 "Aktive und Batch-Segmentierung, Clustering und Seriation: Grundlagen der KI" and by the UFA-DFH through the French-German Doktorandenkolleg CDFA 01-18 and by the SFI Sachsen-Anhalt for the project RE-BCI.

\bibliographystyle{plainnat}
\bibliography{many_opt_bib.bib}

\section*{Checklist}

The checklist follows the references.  Please
read the checklist guidelines carefully for information on how to answer these
questions.  For each question, change the default \answerTODO{} to \answerYes{},
\answerNo{}, or \answerNA{}.  You are strongly encouraged to include a {\bf
justification to your answer}, either by referencing the appropriate section of
your paper or providing a brief inline description.  For example:
\begin{itemize}
  \item Did you include the license to the code and datasets? \answerYes{See Section~\ref{gen_inst}.}
  \item Did you include the license to the code and datasets? \answerNo{The code and the data are proprietary.}
  \item Did you include the license to the code and datasets? \answerNA{}
\end{itemize}
Please do not modify the questions and only use the provided macros for your
answers.  Note that the Checklist section does not count towards the page
limit.  In your paper, please delete this instructions block and only keep the
Checklist section heading above along with the questions/answers below.

\begin{enumerate}

\item For all authors...
\begin{enumerate}
  \item Do the main claims made in the abstract and introduction accurately reflect the paper's contributions and scope?
    \answerYes
  \item Did you describe the limitations of your work?
    \answerYes
  \item Did you discuss any potential negative societal impacts of your work?
    \answerNA
  \item Have you read the ethics review guidelines and ensured that your paper conforms to them?
    \answerYes
\end{enumerate}

\item If you are including theoretical results...
\begin{enumerate}
  \item Did you state the full set of assumptions of all theoretical results?
    \answerYes{In the theorem statements.}
	\item Did you include complete proofs of all theoretical results?
    \answerYes{In the appendices.}
\end{enumerate}

\item If you ran experiments...
\begin{enumerate}
  \item Did you include the code, data, and instructions needed to reproduce the main experimental results (either in the supplemental material or as a URL)?
    \answerYes
  \item Did you specify all the training details (e.g., data splits, hyperparameters, how they were chosen)?
    \answerYes
	\item Did you report error bars (e.g., with respect to the random seed after running experiments multiple times)?
    \answerYes
	\item Did you include the total amount of compute and the type of resources used (e.g., type of GPUs, internal cluster, or cloud provider)?
    \answerNA
\end{enumerate}

\item If you are using existing assets (e.g., code, data, models) or curating/releasing new assets...
\begin{enumerate}
  \item If your work uses existing assets, did you cite the creators?
    \answerNA
  \item Did you mention the license of the assets?
    \answerNA
  \item Did you include any new assets either in the supplemental material or as a URL?
    \answerNA
  \item Did you discuss whether and how consent was obtained from people whose data you're using/curating?
    \answerNA
  \item Did you discuss whether the data you are using/curating contains personally identifiable information or offensive content?
    \answerNA
\end{enumerate}

\item If you used crowdsourcing or conducted research with human subjects...
\begin{enumerate}
  \item Did you include the full text of instructions given to participants and screenshots, if applicable?
    \answerNA
  \item Did you describe any potential participant risks, with links to Institutional Review Board (IRB) approvals, if applicable?
    \answerNA
  \item Did you include the estimated hourly wage paid to participants and the total amount spent on participant compensation?
    \answerNA
\end{enumerate}

\end{enumerate}


\newpage
\appendix

\section{Cumulative regret proofs}
\label{app:regret_proofs}

\subsection{Upper Bound}
\begin{proof}[Proof of Theorem~\ref{thm:UBcumulative}] We denote by $\cL$ the set of arms sampled from the reservoir such that $|\cL| = L$. We also denote by $\cLstar = \{a\in\cL:\ a\in\cAstar\}$ the set of optimal arms in $\cL$ and by $\Lstar = |\cLstar|$ its cardinality. Note that these quantities are all random.

Because of the choice of $L = \big\lceil 4 \log(T)/(\pstar\gamma^2) \big\rceil$, we know that  with high probability there is at least a proportion of $\gamma \pstar$ optimal arms in $\cL$. Precisely, if we denote this favorable event by $\cE = \{\Lstar/L \geq (1-\gamma) \pstar\}$ then by Chernoff's inequality (see Lemma~\ref{lem:chernoff}), we have
\[
\P (\cE^c)= \P\big(\Lstar/L <(1-\gamma) \pstar \big) \leq e^{-\frac{\gamma^2}{4}L\pstar }\leq \frac{1}{T}\,.
\]
We can decompose the regret given this event and its complement:
\begin{align*}
\E[R(T)] &= \E\left[ \sum_{a\in \cL} (\mustar -\mu_a) \E[N_a^T|\cL] \ind_\cE \right]   +T \P(\cE^c)\\
&\leq \E\left[ \sum_{a\in \cL/\cLstar} \Delta_a \E[N_a^T|\cL] \ind_\cE \right] +1\,.
\end{align*}
We now follow the classical proof of UCB-type strategies to upper-bound the number of times a sub-optimal is pulled. From now on, we fix a set of sampled arms $\cL$. Fix an $a \in \cL \setminus \cLstar$. We have
\begin{align*}
  \E[N_a^T|\cL] \leq 1 + \sum_{t=L+1}^T \P(\forall b \in \cLstar,\, U_{t-1}^b \leq \mustar|\cL) +  \P(a_t=a,\, U_{t-1}^a \geq   \mustar|\cL)\,.
\end{align*}
For the first term in the summation we use the fact that there are many optimal arms. Precisely, using Hoeffding's inequality, we have
\begin{align*}
  \P(\forall b \in \cLstar,\, U_{t-1}^b \leq \mustar|\cL) &\leq \P\Bigg(\forall b \in \cLstar, \exists n \in [T]:\ \hmu_{b,n} \\
  &\qquad+  \sqrt{\frac{\gamma^2(1-\gamma)^{-1}/4+\log(\pi^2/6) +2 \log(n)}{2n}} \leq \mustar\Bigg |\cL\Bigg)\\
  &\leq \prod_{b\in \cLstar} \left(\sum_{n=1}^T \frac{1}{n^2} e^{-\gamma^2(1-\gamma)^{-1}/4-\log(\pi^2/6)}\right)\\
  &=e^{-\frac{\gamma^2}{4}(1-\gamma)^{-1}\Lstar }\,.
\end{align*}
For the second term we proceed as usual. Let 
\[n_0= \inf\left\{n \in \mathbb{N}: \sqrt{\frac{\gamma^2(1-\gamma)^{-1}/4+\log(\pi^2/6) +2 \log(n)}{2n}} \leq \Delta/2
\right\}  \]
be such that pulling any arm $a\in\cAsub$ more than $n_0$ times is a small probability event. Note that thanks to Lemma~\ref{lem:inversion_log}
\[
n_0 \leq 4\frac{(1-\gamma)^{-1}  + \log\big(24(1-\gamma)^{-1}/\Delta^2\big)}{\Delta^2}+1\,.
\]
Then, using again Hoeffding's inequality for an arm $a\in\cL \setminus \cLstar$, we obtain
\begin{align*}
  \sum_{t=L+1}^T \P(a_t=a,\, U_{t-1}^a \geq   \mustar|\cL)&\leq \sum_{n=n_a+1}^T \P(\hmu_{a,n} - \mu \geq \Delta/2) + n_0\\
  &\leq \sum_{n\geq 1} e^{-n\Delta^2 /2}+n_0 \leq n_0 +\frac{2}{\Delta^2}\,.
\end{align*}
Collecting the previous inequalities we can conclude for $T\geq 2$
\begin{align}
  \E[R(T)] &\leq \E\left[ \sum_{a\in \cL/\cLstar}  T e^{-\gamma^2(1-\gamma)^{-1}\Lstar/4 } \ind_\cE+ 1 + \Delta n_0 +\frac{2}{\Delta} \right] +1\nonumber\\
  &\leq  \E\left[ \sum_{a\in \cL/\cLstar}  T e^{-\gamma^2L/4 } \ind_\cE+ 1 + \Delta n_0 +\frac{2}{\Delta} \right] +1\nonumber\\
  &\leq L\left(2+ \Delta n_0 + \frac{2}{\Delta}\right) +1\nonumber\\
  & \leq  \frac{8\log(T)}{\pstar \Delta\gamma^2 }\Big(10(1-\gamma)^{-1} + 4\log\big(24(1-\gamma)^{-1}/\Delta^4\big)\Big)+1\,.\label{eq:final_ub_regret}
\end{align}
\end{proof}


\subsection{Lower Bound}

We denote by $\Ber(p)$ the Bernoulli distribution of parameter $p$. The Kullback-Leibler (KL) divergence between probability distributions $P$ and $Q$ is denoted by $\KL(P,Q)$. In particular, the KL divergence between two Bernoulli distributions $\Ber(p)$ and $\Ber(q)$ is 
\[
\kl(p,q) = \KL\!\big(\!\Ber(p),\Ber(q)\big) = p\log\left(\frac{p}{q}\right) +(1-p)\log\left(\frac{1-p}{1-q}\right).
\]

\begin{proof}[Proof of Theorem~\ref{thm:LBcumulative2}] We fix a partition of the reservoir $\cA = \cA_1 \cup \cA_2 \cup \cA_3$ and set $\pstar$ the probability to sample an arm in $\cA_1$, $\cA_2$ and $1-2\pstar$ the probability to sample an arm in $\cA_3$. We define two bandits problems associated with this reservoir. The bandit problem $\nu$ where the arms in $\cA_1$ have probability distribution $\Ber(1/2)$, the arm in $\cA_2$ and $\cA_3$ have probability distribution $\Ber(1/2-\Delta)$. The second bandit problem $\nu'$ is such that the arms in $\cA_1$ have probability distribution $\Ber(1/2)$, the arms in $\cA_2$ have probability distribution $\Ber(1/2+\Delta)$ and the arms in $\cA_3$ have probability distribution $\Ber(1/2-\Delta)$. We denote by $\E_\nu$ respectively $\E_{\nu'}$ the expectation under the bandit problem $\nu$ respectively $\nu'$.

Let $N_{\cA_i}^T = \sum_{t=1}^T \ind_{\{a_t\in \cA_i\}} $ be the number of times an arm in $\cA_i$ is pulled. Note that since the arms in $\cA_2$ and $\cA_3$ are indistinguishable for the agent in the problem $\nu$, it holds 
\[
\E_{\nu}[N_{\cA_2}^T]  =  \frac{\pstar}{1-\pstar}\E_{\nu}[N_{\cA_2}^T + N_{\cA_3}^T].
\]
Let $I^t$ be the information available by the agent at time $t$, i.e. the collection  of collected rewards and arms pulled. We denote by $\P^{I^t}_\nu$ respectively $\P^{I^t}_{\nu'}$ the distribution of this random variable under the bandit problem $\nu$ respectively $\nu'$. Thanks to the chain rule and the above remark we can upper bound the Kullback-Leibler divergence between these two probability distributions
\begin{align}
    \KL(\P^{I^T}_\nu, \P^{I^T)}_{\nu'}) &= \kl(1/2-\Delta,1/2+\Delta) \E_\nu[N_{\cA_2}^T]\nonumber\\
    &=\kl(1/2-\Delta,1/2+\Delta) \frac{\pstar}{1-\pstar} \E_{\nu}[N_{\cA_2}^T + N_{\cA_3}^T]\nonumber\\
    &\leq 22\pstar \Delta^2 \E_{\nu}[N_{\cA_2}^T + N_{\cA_3}^T]  = 22 \pstar \Delta \E_{\nu}\big[R(T)\big]\,,\label{eq:ub_KL_lb_regret}
    \end{align}
where in the last inequality we used that $\pstar \leq 1/4$ and 
\[
\kl(1/2-\Delta,1/2+\Delta) = 2\Delta \log\left(1+\frac{2\Delta}{1/2-\Delta}\right) \leq \frac{4\Delta^2}{1/2-\Delta} \leq 16 \Delta^2.
\]
We assume that 
\[\E_{\nu} \big[R(T)\big] = \Delta \big(T-\E_{\nu}[N_{\cA_1}^T]\big) \leq \sqrt{T},\qquad  \E_{\nu'} \big[R(T) \big] = \Delta\E_{\nu'}[N_{\cA_1}^T] +2 \Delta \E_{\nu'}[N_{\cA_3}^T]\leq \sqrt{T},\] otherwise the result is trivially true. In particular this implies that
\begin{equation}
\label{eq:lb_regret_bound_counts}    1-\sqrt{\frac{1}{\Delta^2T}}\leq \frac{\E_{\nu}[N_{\cA_1}^T]}{T}\ \qquad\frac{\E_{\nu'}[N_{\cA_1}^T]}{T} \leq \sqrt{\frac{1}{\Delta^2 T}} \,.
\end{equation}
Using the contraction of the entropy (see \citet{garivier2019explore}), the inequality $\kl(x,y) \geq x\log(1/y) -\log(2)$ then~\eqref{eq:lb_regret_bound_counts}, we obtain 
\begin{align*}
    \KL(\P^{I^T}_\nu, \P^{I^T)}_{\nu'}) &\geq  \kl\left( \E_{\nu}[N_{\cA_1}^T]/T, \E_{\nu'}[N_{\cA_1}^T]/T\right)\\
    &\geq \frac{\E_{\nu}[N_{\cA_1}^T]}{T} \log\!\left(  \frac{T}{\E_{\nu'}[N_{\cA_1}^T]} \right) -\log(2)\\
    &\geq \frac{1}{2} \left(1-\sqrt{\frac{1}{\Delta^2 T}}\right) \log(\Delta^2 T) -\log(2)\,.
\end{align*}
The previous inequality with the fact that the Kullback-Leibler divergence is positive yields
\begin{equation}
\KL(\P^{I^T}_\nu, \P^{I^T)}_{\nu'}) \geq \frac{2}{3}\log(\Delta^2 T/16)^+\,.\label{eq:lb_KL_lb_regret}
\end{equation}
Indeed if $\Delta^2 T /16 \leq 1$ then \eqref{eq:lb_KL_lb_regret} is trivially true. In the other case we have 
\begin{align*}
    \frac{1}{2} \left(1-\sqrt{\frac{1}{\Delta^2 T}}\right) \log(\Delta^2 T) -\log(2) &\geq \frac{3}{8} \log(\Delta^2 T) - \frac{1}{4} \log(16)\\
    &\geq  \frac{3}{8} \log(\Delta^2 T /16)
\end{align*}
Combining~\eqref{eq:ub_KL_lb_regret} and~\eqref{eq:lb_KL_lb_regret} allows us to conclude 
\[
\E_{\nu}\big[R(T)\big] \geq \frac{1}{60} \frac{\log(\Delta^2 T/16)^+}{\pstar \Delta}.
\]

\end{proof}

\subsection{Impossibility of adaptation to $\pstar$}
\begin{proof}[Proof of Theorem~\ref{thm:unkownpcumulative}]

Consider $\Delta \in (0, 1/4)$ and the following two definitions of two reservoir distributions:
\begin{itemize}
    \item The reservoir distribution $\Nu_0$ characterised by $p_1 = \pstar$ and $p_2 = 1-\pstar$ and $\nu_1 = \mathcal B(1/2)$ and $\nu_2 = \mathcal B(1/2 - \Delta)$.
    \item The reservoir distribution $\Nu_1$ characterised by  $p_1 = \qstar$, $p_2 = \pstar$ and $p_3 = 1- \qstar - \pstar$ and $\nu_1 = \mathcal B(1/2 +\Delta)$ and $\nu_2 = \mathcal B(1/2 )$ and $\nu_3 = \mathcal B(1/2 - \Delta)$.
\end{itemize}

Note that the Bernoulli distribution is completely characterised by its mean and so we can use the mean to characterise the distribution. Let $\tilde \mu = (\tilde \mu_j)_{j \leq T}$ be $T$ i.i.d.~means corresponding to $T$ i.i.d.~distributions sampled according to the reservoir distribution $\Nu_1$. Note that $\tilde \mu_j \in \{1/2-\Delta, 1/2, 1/2+\Delta\}$. Write also $\tilde \mu' = (\tilde \mu_j')_{j \leq T}$ for the vector of means such that $\tilde \mu_j' = \tilde \mu_j$ if $\tilde \mu_j' \in\{1/2-\Delta, 1/2\}$, and $\tilde \mu_j' = 1/2 - \Delta$ otherwise. Note that then, we have that $(\tilde \mu_j')_{j \leq T}$ are $T$ i.i.d.~means corresponding to $T$ i.i.d.~distributions sampled according to the reservoir distribution $\Nu_0$, by definition of $\Nu_0$. Write $\mathbb E_{\Nu_1}$ for the expectation according to the distribution of $\tilde \mu$, i.e.~according to $\Nu_1^{\otimes T}$, and $\mathbb E_{\Nu_0}$ for the expectation according to the distribution of $\tilde \mu'$, i.e.~according to $\Nu_0^{\otimes T}$. 

Consider an algorithm $\mathfrak A$ and a bandit problem involving Bernoulli distributions characterised by a vector of means $m = (m_j)_{j \leq T}$. Write $\mathbb P_m^{\mathfrak A}$ for the distribution of the samples obtained by the algorithm run on this problem, and $\mathbb E_m^{\mathfrak A}$ the associated expectation. Consider now another Bernoulli bandit problem characterised by the means $m' = (m'_j)_{j \leq T}$. We have because of the chain rule 
$$\KL(\mathbb P_{m'}^{\mathfrak A}, \mathbb P_{m}^{\mathfrak A})= \sum_{j \leq T} \mathbb E_{m'}^{\mathfrak A}[T_j] \kl(m_j', m_j),$$
where $\mathbb E_{m'}^{\mathfrak A}$ is the expectation according to problem $m'$ on which algorithm $\mathfrak A$ is used, and where $T_j$ is the number of times arm $j$ is sampled at time $T$.

From our assumption on $\mathfrak{A}$ we have that  $\E_{\Nu_0} \left[ R(T)\right] \leq \frac{\log(T)}{\pstar\Delta}$. Now, we can obtain
\begin{align}
    \KL(\mathbb E_{\Nu_0} \mathbb P_{\tilde \mu'}^{\mathcal A}, \mathbb E_{\Nu_1} \mathbb P_{\tilde \mu}^{\mathcal A}) &= \KL(\mathbb E_{\Nu_1} \mathbb P_{\tilde \mu'}^{\mathfrak A}, \mathbb E_{\Nu_1} \mathbb P_{\tilde \mu}^{\mathfrak A}) \nonumber\\
    &\leq \mathbb E_{\Nu_1} \Bigg[\KL(\mathbb P_{\tilde \mu'}^{\mathfrak A}, \mathbb P_{\tilde \mu}^{\mathfrak A})\Bigg] = \mathbb E_{\Nu_1} \Bigg[\sum_{j \leq T} \mathbb E_{\tilde \mu'}^{ \mathfrak A}[T_j] \kl(\tilde \mu_j', \tilde \mu_j)\Bigg] \nonumber\\
    &\leq  \mathbb E_{\Nu_1} \Bigg[\sum_{j \leq T} \mathbb E_{\tilde \mu'}^{\mathfrak A}[T_j] \frac{\Delta^2}{16} \mathbf 1\{\tilde \mu_j = 1/2+\Delta\} \Bigg] \nonumber\\
    &=  \mathbb E_{\Nu_0} \Bigg[\sum_{j \leq T} \mathbb E_{\tilde \mu', \mathfrak A}[T_j] \frac{\Delta^2}{16} \mathbf 1\{\tilde \mu_j' = 1/2-\Delta\} \frac{\qstar}{1 - \pstar} \Bigg]\nonumber \\
    &= \frac{\qstar\Delta}{8} \mathbb E_{\Nu_0} \left[ R(T)\right] \leq \frac{c\qstar}{8\pstar}\log(T) \leq \frac{1}{2}\log(T),\label{eq:klineq2}
\end{align}


where the last equality follows since by definition of $\Nu_0, \Nu_1$, conditionally on $\tilde \mu_j' = 1/2-\Delta$, the probability that $\tilde \mu_j = 1/2+\Delta$ is $\frac{\qstar}{1 - \pstar} \leq 2\qstar$, and otherwise it is $0$. And where the final inequality comes from our assumption $\pstar > \frac{c\qstar}{4}$.

Consider the event, 

\[E := \Bigg\{\sum_{j \leq T} T_j  \mathbf 1\{ \tilde \mu_j' = 1/2\}> T/2 
\Bigg\}\;.\]
Note that on $\Nu_0$, we have $\mu^* = \frac{1}{2}$. Thus, on $\Nu_0$ the event $E^C$ will signify a regret greater than $\frac{T\Delta}{2}$, similarly on $\Nu_1$ the event $E$ signifies a regret greater than $\frac{T\Delta}{2}$. Thus, 
\begin{equation}\label{eq:E}
E^C \subset \left\{R_{\Nu_0}(T) \geq \frac{T\Delta}{2}\right\}\;, \qquad E \subset \left\{R_{\Nu_1}(T) \geq \frac{T\Delta}{2}\right\}\;. \end{equation}
Where $R_{\Nu_0}(T)$ and $R_{\Nu_1}(T)$ denote the regret of the algorithm on $\Nu_0$ and $\Nu_1$ respectively. Now from our assumption upon $\mathfrak{A}$ we have that $\mathbb E_{\Nu_0} R(T) \leq \frac{c\log(T)}{\pstar \Delta}$, therefore Equation~\eqref{eq:E} leads to,

\begin{equation}\label{eq:regretonEC}
\mathbb E_{\Nu_0} \mathbb P_{\tilde \mu'}^{\mathfrak A}\left(E^C\right) \leq \frac{c\log(T)}{\pstar \Delta} \times \frac{2}{T\Delta}\;.
\end{equation}
and in addition we also have,
\begin{equation}\label{eq:regretonE}
\mathbb E_{\Nu_1} R(T) \geq  \mathbb E_{\Nu_1} \mathbb P_{\tilde \mu}^{\mathfrak A}\left(E\right)\times \frac{T\Delta}{2}\;.
\end{equation}
Now, using the Bretagnolle-Huber's inequality (see Theorem 14.2 by \citet{lattimore2020bandit}) in combination with~\eqref{eq:klineq2} we obtain  
\begin{align*}
    \mathbb E_{\Nu_0} \mathbb P_{\tilde \mu'}^{\mathfrak A}(E^C) + \mathbb E_{\Nu_1} \mathbb P_{\tilde \mu}^{\mathfrak A}(E) &\geq \frac{1}{2} \exp\!\!\Bigg(-\KL(\mathbb E_{\Nu_1} \mathbb P_{\tilde \mu'}^{\mathfrak A}, \mathbb E_{\Nu_1} \mathbb P_{\tilde \mu}^{\mathfrak A})\Bigg)\\
    &\geq \frac{1}{2\sqrt{T}}\,.
\end{align*}
This result in combination with Equation~\eqref{eq:regretonEC} gives the following,

\begin{equation}\label{eq:boundR1}\mathbb E_{\Nu_1} \mathbb P_{\tilde \mu}^{\mathfrak A}(E) \geq \frac{1}{2\sqrt{T}} - \frac{2c\log(T)}{\pstar T\Delta^2} \geq \frac{1}{4\sqrt{T}}\end{equation}
where the final inequality comes from our assumption $T \geq 4\left(\frac{c\log(T)}{\pstar\Delta^2}\right)^2$. Finally our result follows from combination of Equation~\eqref{eq:regretonEC} and Equation~\eqref{eq:boundR1}.

\end{proof}
\section{Best-arm identification proofs}
\label{app:BAIproofs}

\subsection{Upper Bound}
\begin{proof}[Proof of Theorem~\ref{thm:UBsimple}]

\textbf{Proof-specific notations and preliminary considerations.} At round $i$, write $K_i = |\mathcal A_i|$ and write $p_i$ for the proportion of optimal arms in $\mathcal A_i$, namely
$$p_i = |\mathcal A_i \cap \cA^*|/|\mathcal A_i|.$$
We also write $M_i$ for the number of optimal arms in $\mathcal A_i$ such that $\hat \mu_i(a) \geq \mu^* - \Delta/2$, namely
$$M_i = \big|\{a\in \mathcal A_i \cap \cA^*:\hat \mu_i(a) \geq \mu^* -  \Delta/2\}\big|\,,$$
and $N_i$ for the number of sub-optimal arms in $\mathcal A_i$ such that $\hat \mu_i(a) \geq \mu^* -\Delta/2$, namely
$$N_i = \big|\{a \in \mathcal A_i \cap \cA_{sub}:\hat \mu_i(a) \geq \mu^* - \Delta/2\}\big|\,.$$
Note that by definition
$$K_{i+1} = \left(1 \lor \left\lfloor\frac{K_i}{2}\right\rfloor\right) + \left\lfloor\frac{K_i}{4}\right\rfloor\,.$$ 
Therefore the following bounds holds 
\begin{equation}
\label{eq:bounds_K_i}
\left(\left(\frac{3}{4}\right)^i K_1\right) \lor 1 \geq K_i \geq \left(\frac{1}{2}\right)^i K_1 -4\,.
\end{equation}

We write $I$ for the smallest index $i$ such that $K_i = 1$ and will not investigate what happens at rounds $i > I$. By the upper bound~\eqref{eq:bounds_K_i} on $K_i$ it holds $I \leq \log_{4/3}(K_1) \leq \log_{4/3}(T)$. Note that since $\log_{4/3}(T) = \bar c \log T$, the algorithm terminates with a set containing just one arm.

\textbf{Step 1: Introduction of high-probability events of interest.} We define the constant 
\[c= \frac{\bar c}{10} .\]

We define $j^*$ as the largest $j$ smaller than or equal to $I$ such that 
$$K_j  \geq  c T\Delta^2/(2 \log T).$$
Note that such $j^*$ exists since $K_1\geq \bar c T/(2\log T)$, and since $K_I = 1$. We prove below the following upper bound on $j^*$.
Take any round $i$. Note that for any $k$, conditionally on $\cA_i$, by Hoeffding's inequality, for any $a\in\cA_i$
\begin{equation}\label{eq:hoeffSR}
    \P\Big(\big|\hat \mu_{i}(a) - \mu_{i}(a)\big| \geq \Delta/2\Big|\cA_i\Big) \leq 2\exp(-  \Delta^2 t_i/2)  = q_i,
\end{equation}
where $\mu_{i}(a)$ is the true mean associated with arm $a$. We now state the following technical lemma proved below.
\begin{lemma}\label{lem:tech}
Assume that $\pstar \leq 1/2$, and consider $I\geq i \geq j^*$. Under the assumptions of the theorem, we have 
\begin{align}
    q_i^{-1/2} &\geq 200 \geq e^2 - 1\,, \label{eq:first_term_sr}\\
    \Delta^2 t_i/4 &\geq \log 2\label{eq:second_term_sr}\,.
\end{align}
\end{lemma}

We define for $i \geq j^*$ and 
$\bar p_i:=\left(\frac{\pstar}{6} (5/4)^{i-j^*} \land (1/2)\right)$, the event
$$\xi_i = \left\{p_i > \bar p_i \right\}.$$
Consider from now on $i \geq j^*$.

\paragraph{Step 2: Lower bound on $M_i$ conditional to $\xi_{i}$.} 

We have by definition of $M_i$:
$$M_i = \sum_{a \in \mathcal A_i \cap \cA^*} \mathbf 1\{\hat \mu_i(a) \geq \mu^* -  \Delta/2\},$$
where by Equation~\eqref{eq:hoeffSR}, and conditionally on $\mathcal A_i$, the $\mathbf 1\{\hat \mu_i(a) \geq \mu^* - \Delta/2\}$ are independent and dominate stochastically $\mathcal B(1 - q_i)$, for any $a \in \cA_i\cap \cA^*$. And so conditionally on $\mathcal A_i$, we have that $M_i$ stochastically dominates $\mathcal B(K_i p_i, 1 -q_i)$. 
And so by Chernoff's inequality, for any $x\geq \sqrt{q_i}$:
\begin{align*}
\mathbb P(M_i -  p_i K_i (1- q_i)\leq - x p_i K_i|\mathcal A_i) &\leq \left[\frac{e^{x/q_i}}{(1+x/q_i)^{1+x/q_i }}\right]^{K_i p_i q_i}\\
&\leq \exp\Big[xK_i p_i - \log(1+x/q_i)(K_i p_i q_i + x K_i p_i)\Big]\\
&\leq (1+x/q_i)^{-xK_i p_i /2}.
\end{align*}
as for $i > j^*$ we have $\log(1+x/q_i) > 2$, see Lemma \ref{lem:tech}.

So that for $x \geq  \sqrt{q_i}$
$$\mathbb P(M_i \leq K_i p_i(1- 2x)|\mathcal A_i) \leq \exp\Big(-   x\Delta^2 t_i K_i p_i/16\Big),$$
since $\log(q_i^{-1}) = \Delta^2 t_i/2 - \log 2 \geq \Delta^2 t_i/4$ for $I \geq i \geq j^*$ - see Lemma~\ref{lem:tech}.

And so since $p_i \geq \frac{\pstar}{6}$ on $\xi_i$
\begin{equation}\label{eq:M}
\mathbb P(M_i \leq p_i K_i (1- 2x)|\xi_i) \leq \exp\left(- \bar c' x \pstar \Delta^2 T / \log T\right) := u.
\end{equation}
where $\bar c' = \bar c  /96$ and recalling $t_i = \lfloor \bar c T / \big(K_i \log(T)\big) \rfloor$.




\paragraph{Step 3: Upper bound on $N_i$ conditional to $\xi_{i}$.}

We have by definition of $N_i$:
$$N_i = \sum_{a \in \mathcal A_i \cap \cA_{sub}} \mathbf 1\{\hat \mu_i(a) \geq \mu^* -  \Delta/2\},$$
where by Equation~\eqref{eq:hoeffSR}, and conditionally on $\mathcal A_i$, the $\mathbf 1\{\hat \mu_i(a) \geq \mu^* - \bar \Delta/2\}$ are independent and are stochastically dominated by $\mathcal B(q_i)$, for any $a \in \cA_i\cap \cA_{sub}$. And so conditionally on $\mathcal A_i$, we have that $N_i$ is stochastically dominated by $\mathcal B(K_i, q_i)$. 
And so by Chernoff's inequality for any $x\geq 2$:
$$\mathbb P(N_i -  K_i q_i \geq  x K_i|\xi_i) \leq \left[\frac{e^{x/q_i}}{(1+x/q_i)^{1+x/q_i}}\right]^{K_i q_i} \leq (1+x/q_i)^{-xK_i /2},$$
similar to Step 2. 


So that for $x \geq \sqrt{q_i}$
$$\mathbb P(N_i \geq 2 K_i x |\mathcal A_i) \leq \exp\Big(-  x\Delta^2 t_i K_i /16\Big),$$
as in Step 2.

And so similar to in Step 2:
\begin{equation}\label{eq:N}
\mathbb P(N_i \geq 2x K_i |\xi_i) \leq \exp\left(- \bar c' x\Delta^2 T / \log T\right) \leq u.
\end{equation}





\paragraph{Step 4: Bound on the probability of $\xi_i$ and conclusion.}

First we have -- since we add $K_{j^*-1}/4 = K_{j^*}/3$ fresh arms to the set $\mathcal A_{j^*}$ - that 
$$\left\{\left|\sum_{a \in \cA_{j^*}} \mathbf 1\{a \in \cA^*\} - \frac{1}{3}\pstar K_{j^*}\right| \leq \frac{1}{6} \pstar K_{j^*}\right\} \subset \xi_{j^*},$$
where it holds that $\mathbf 1\{a \in \cA^*\} \sim \mathcal B(p^*)$ for the fresh arms and $|\cA_{j^*}| = K_{j^*}$. And so by Chernoff's inequality:
\begin{equation}\label{eq:xijs}
\mathbb P(\xi_{j^*}) \geq 1 - 2\exp(-\pstar K_{j^*}/10) \geq 1 - 2\exp\left(- c \frac{\pstar T \Delta^2}{20\log T}\right) =: 1-v,    
\end{equation}
by 
definition of $j^*$.

Now consider $i >j^*$, let,
$$\xi^{'}_{i} = \left\{p_{i+1} \geq \frac{5}{4}p_i \wedge \frac{1}{2}\right\}.$$



\begin{lemma}\label{lem:xipxipp}
Assume that $2x \leq 1/100$. We have for $I\geq i>j^*$:
$$\xi_i'' := \left\{M_i > p_i K_i (1- 2x)\right\}\cap \left\{N_i < 2 x K_i\right\} \subset \xi_i'.$$ 
\end{lemma}

Note also that
$$\mathbb P(\xi_i''|\xi_i) \geq 1-2u,$$
by Equations~\eqref{eq:M} and~\eqref{eq:N}, so that by Lemma~\ref{lem:xipxipp}
\begin{equation}\label{eq:probxip}
    \mathbb P(\xi_i'|\xi_i) \geq 1-2u.
\end{equation}

By induction it holds that for any $1 \leq m \leq I - j^*$
$$\xi_{j^*} \cap \bigcap_{j^*<i\leq j^*+m} \xi_i' \subset  \bigcap_{j^* \leq i\leq j^*+m} \xi_i,$$
so that by Equations~\eqref{eq:xijs} and~\eqref{eq:probxip} 
$$\mathbb P\left( \bigcap_{j^*\leq i\leq j^*+m} \xi_i\right) \geq (1 - v)(1 - 2u)^{m}\geq 1 - v - 2um.$$

In particular using the previous inequality for $m=I-j^*$ and since $I \leq \log T$ it holds
$$\mathbb P\left(\bigcap_{j^*\leq i\leq I} \xi_i\right) \geq 1 - v - 2u\log T.$$
Since $K_I = 1$, and since by definition of the $\xi_i$ we know that on $\xi_I$ we have that the only arm in $\cA_I$ is optimal, this concludes the proof - taking $x = 1/200$, which is compatible with $x\geq \sqrt{q_i}$ as $q_i \leq 1/200^2$ by Lemma \ref{lem:tech}.





\end{proof}

We prove now successively, 
Lemma~\ref{lem:tech}, Lemma~\ref{lem:xipxipp} used in the proof of Theorem~\ref{thm:UBsimple}.




\begin{proof}[Proof of Lemma~\ref{lem:tech}] Note first that for $I\geq i \geq j^*$ we have
$$K_{i+1} = \lfloor K_i/2\rfloor \lor 1 + \lfloor K_i/4\rfloor \leq \frac{3K_i}{4} \lor 1.$$
So that for any $0 \leq m<I-j^*$ we have by definition of $I$ as the first index such that $K_I = 1$
\begin{equation}\label{eq:UPKi}
    K_i \leq K_{j^*} (3/4)^{m}.
\end{equation}

Also for any $i$ such that $K_i \geq 4$
$$K_{i+1} \geq K_i/2,$$
and for any $i$ such that $K_i < 4$, we have
$$K_{i+1} = 1,$$
so that for any $0 \leq m<I-j^*$ we have
$$K_i \geq K_{j^*} (i/2)^{m}.$$ 

\paragraph{Inequality~\eqref{eq:first_term_sr}:} We therefore have for $I > i \geq j^*$ and by Equation~\eqref{eq:UPKi} 
\begin{align*}
    q_i^{-1/2} = 2^{-1/2}\exp(\Delta^2 t_i/4) &\geq 2^{-1/2}\exp\left(\bar c \frac{ \Delta^2 T}{2K_{j^*}\log(T) } \right)\;,\\
    &\geq  2^{-1/2} \exp(10)\geq 200 \geq e^2 - 1
\end{align*}

\paragraph{Inequality~\eqref{eq:second_term_sr}:}
We have,
\[q_i = \exp(-\Delta^2 t_i/2)\;,\]
thus by inequality~\eqref{eq:first_term_sr} we have
$$\exp(\Delta^2 t_i/4) \geq \sqrt{2} (e^2 - 1),$$
so that 
$$\Delta^2 t_i/4 \geq \log 2.$$


\end{proof}

\begin{proof}[Proof of Lemma~\ref{lem:xipxipp}] Let $i$ such that $I\geq i>j^*$. Note that on ${\xi_i}''$, we have $M_i>0$ so that $p_i>0$.

\paragraph{First case: $0<p_i \leq 2/5$.} Assume first that $p_i \leq 2/5$. On $\xi_i''$ we have that
$$M_i > p_i K_i (1- 2x),$$
and
$$N_i< 2K_i x,$$
so that 
$$M_i+N_i < p_i K_i + 2K_ix \leq (2/5)K_i + K_i/100 \leq K_i/2.$$
since $2x \leq 1/100$ for $i \geq j^*$ - see Lemma~\ref{lem:tech}. And so all $M_i$ arms of $\{a\in \mathcal A_i \cap \cA^*:\hat \mu_i(a) \geq \mu^* - \bar \Delta/2\}$ are going to be in $\cA_{i+1}$. This implies -- as in this case $K_i \geq 2$ otherwise we cannot have $0<p_i\leq 2/5$ -- that
$$p_{i+1} \geq \frac{M_i}{K_{i+1}} = \frac{M_i}{1\lor \lfloor K_{i}/2\rfloor + \lfloor K_{i}/4\rfloor} \geq \frac{4}{3}(1- 2x) p_i > \frac{5}{4} p_i,$$
as $2x \leq 1/100$.

\paragraph{Second case: $p_i > 2/5$.}
Assume now that $p_i > 2/5$. On $\xi_i''$ we have that
$$M_i > p_i K_i (1- 2x) \geq \frac{198}{500} K_i,$$
and
$$N_i< 2K_i x \leq K_i/100,$$
since $2x \leq 1/100$ for $I\geq i > j^*$ -- see Lemma~\ref{lem:tech}. Since $198/500 + 1/100 = 203/500 < 1/2$ this implies that at least $\frac{199}{500} K_i$ from the arms in $\{a\in \mathcal A_i \cap \cA^*:\hat \mu_i(a) \geq \mu^* - \bar \Delta/2\}$ are going to be in $\cA_{i+1}$. So that
$$p_{i+1} \geq \frac{M_i}{K_{i+1}} = \frac{M_i}{1\lor \lfloor K_{i}/2\rfloor + \lfloor K_{i}/4\rfloor} \geq \frac{4}{3} \times \frac{198}{500} = \frac{66}{125}>1/2.$$
This concludes the proof.
\end{proof}

\subsection{Lower Bound}
\begin{proof}[Proof of Theorem~\ref{thm:LBsimple}] We consider a similar setting to that in the proof of Theorem~\ref{thm:unkownpcumulative} although with a slightly different construction of $\Nu_0, \Nu_1$.

Consider the following two reservoir distributions:
\begin{itemize}
    \item The reservoir distribution $\Nu_0$ characterised by $p_1 = \pstar$ and $p_2 = 1 - \pstar$ and $\nu_1 = \mathcal B(1/2)$ and $\nu_2 = \mathcal B(1/2 - \Delta)$.
    \item The reservoir distribution $\Nu_1$ characterised by $p_1 = \pstar$ and $p_2 = \pstar$ and $p_3 = 1-2\pstar$ and $\nu_1 = \mathcal B(1/2 + \Delta)$ and $\nu_2 = \mathcal B(1/2)$ and $\nu_3 = \mathcal B(1/2 - \Delta)$.
\end{itemize}

We define $\tilde \mu, \tilde \mu'$, and associated expectations and probabilities as in the proof of Theorem \ref{thm:unkownpcumulative}. Consider also any algorithm $\mathfrak A$. We have by similar calculations as Equation~\eqref{eq:klineq2} the following upper bound on the KL divergence
\begin{align}
    \KL(\mathbb E_{\Nu_0} \mathbb P_{\tilde \mu'}^{\mathfrak A}, \mathbb E_{\Nu_1} \mathbb P_{\tilde \mu}^{\mathfrak A}) &= \KL(\mathbb E_{\Nu_1} \mathbb P_{\tilde \mu'}^{\mathfrak A}, \mathbb E_{\Nu_1} \mathbb P_{\tilde \mu}^{\mathfrak A}) \nonumber\\
    &\leq \mathbb E_{\Nu_1} \Bigg[\KL(\mathbb P_{\tilde \mu'}^{\mathfrak A}, \mathbb P_{\tilde \mu}^{\mathfrak A})\Bigg] = \mathbb E_{\Nu_1} \Bigg[\sum_{j \leq T} \mathbb E_{\tilde \mu'}^{ \mathfrak A}[T_j] \kl(\tilde \mu_j', \tilde \mu_j)\Bigg] \nonumber\\
    &\leq  \mathbb E_{\Nu_1} \Bigg[\sum_{j \leq T} \mathbb E_{\tilde \mu'}^{\mathfrak A}[T_j] \frac{\Delta^2}{16} \mathbf 1\{\tilde \mu_j = 1/2+\Delta\} \Bigg] \nonumber\\
    &=  \mathbb E_{\Nu_0} \Bigg[\sum_{j \leq T} \mathbb E_{\tilde \mu'}^{\mathfrak A}[T_j] \frac{\Delta^2}{16} \mathbf 1\{\tilde \mu_j' = 1/2-\Delta\} \frac{\pstar}{1 - \pstar} \Bigg],\label{eq:KLineq}
\end{align}
since by definition of $\Nu_0, \Nu_1$, conditionally on $\tilde \mu_j' = 1/2-\Delta$, the probability that $\tilde \mu_j = 1/2+\Delta$ is $\frac{\pstar}{1 - \pstar}$, and otherwise it is $0$.

By Equation~\eqref{eq:KLineq} and since $\sum_{j \leq T} \mathbb E_{\tilde \mu'}^{\mathfrak A}[T_j] = T$, we have
\begin{align*}
    \KL(\mathbb E_{\Nu_0} \mathbb P_{\tilde \mu'}^{\mathfrak A}, \mathbb E_{\Nu_1} \mathbb P_{\tilde \mu}^{\mathfrak A})  \leq  T \frac{\Delta^2}{16} \frac{\pstar}{1 - \pstar}.
\end{align*}

Now by Bretagnolle-Huber's inequality (see Theorem 14.2 by \citet{lattimore2020bandit}) and for any event $E$
\begin{equation}\label{eq:BH}
\mathbb E_{\Nu_1} \mathbb P_{\tilde \mu}^{\mathfrak A}(E) + \mathbb E_{\Nu_0} \mathbb P_{\tilde \mu'}^{\mathfrak A}(E^C) \geq \frac{1}{2} \exp\Bigg(-\KL(\mathbb E_{\Nu_0} \mathbb P_{\tilde \mu'}^{\mathfrak A}, \mathbb E_{\Nu_1} \mathbb P_{\tilde \mu}^{\mathfrak A})\Bigg).
\end{equation}
Let us write $\ha_T$ for the arm that the algorithm $\mathfrak A$ recommends. Set 
$$E = \{\tilde \mu_{\ha_T} = 1/2\}.$$
Note that on $E$, we make a mistake in prediction for $\tilde \mu$, and that on $E^C$, we make a mistake in prediction for $\tilde \mu'$. We have
$$ \mathbb E_{\Nu_1} \mathbb P_{\tilde \mu}^{\mathfrak A}(E) + \mathbb E_{\Nu_1} \mathbb P_{\tilde \mu'}^{\mathfrak A}(E^C) \geq \frac{1}{2} \exp\Bigg(-T \frac{\Delta^2}{16} \frac{\pstar}{1 - \pstar}\Bigg).$$
This concludes the proof by definition of $E$.
\end{proof}
\section{Technical lemmas}
\label{app:technical_lemmas}

\begin{lemma}(Chernoff bound)
\label{lem:chernoff}
Let $X_1,\ldots,\X_n \sim \Ber(p)$ be n samples from a Bernoulli distribution and $S_n=\sum_{k=1}^n X_n $ their sum. Then for all $\gamma\in[0,1]$ it holds 
\begin{align*}
    \P\left(\frac{S_n}{n} \leq (1-\gamma) p \right) \leq e^{-\frac{\gamma^2 }{4} n p}\,,\\
    \P\left(\frac{S_n}{n} \geq (1+\gamma) p \right) \leq e^{-\frac{\gamma^2 }{4} n p}\,.
\end{align*}

\end{lemma}

\begin{proof}
We prove the first inequality; the second one is similar. If $(1-\gamma)p<0$ or $\gamma = 0$ the inequality is trivially true. Else, because of Chernoff's inequality, we have 
\begin{align*}
    \P\left(\frac{S_n}{n} \leq (1-\gamma) p \right) \leq e^{-n\kl\big((1-\gamma) p,p\big)}\,.
\end{align*}
It remains to remark to conclude that
\begin{align*}
    \kl\big((1-\gamma) p,p\big) \geq  \frac{\gamma^2}{2}p\,,
\end{align*}
where we used the refined Pinsker inequality from \citet{garivier2019explore}, for $0\leq x < y \leq 1$,
\[
\kl(y,x) \geq \frac{1}{2\max_{x\leq q\leq y} q (1-q) } (x-y)^2\geq \frac{1}{2y}(x-y)^2\,. 
\]
For the second inequality we use 
\[
 \kl\big((1+\gamma) p,p\big) \geq  \frac{1}{2(1+\gamma) p} \gamma^2 p^2 \geq \frac{\gamma^2}{4} p\,. 
\]
\end{proof}

\begin{lemma}
\label{lem:inversion_log}
Let $A,B,C\geq 0$ be constants such that $A\geq C$, then for $n_0 = \inf\{ n\geq 1:\ A+B\log(n) \leq n C\}$ we have 
\[
n \leq \frac{A+B\log\big((2(B^2+AC)/C^2\big)}{C}+1\,.
\]
\end{lemma}
\begin{proof}
First let $x_0\geq 1$ be such that $A+B\log(x_0) = C x_0$. It exists since $A+B\log(x)/x \to 0 $ if $x\to \infty$ and since $A\geq C$. In particular, because of the definition of $n_0$ we have $x_0\leq n_0 \leq x_0 +1$. Then note that $A+B\sqrt{x_0} \leq C x_0$. Thus $\sqrt{x_0}$ is smaller than the largest roots of the polynomial $Cy^2-By-A$. Using $\sqrt{a+b}\leq \sqrt{a} +\sqrt{b} $ and $(a+b)^2\leq 2(a^2+b^2)$ we obtain \begin{align*}
    x_0 &\leq \left(\frac{B+ \sqrt{B^2+4AC}}{2C}\right)^2\\
    &\leq 2 \frac{B^2+AC}{C^2}\,.
\end{align*}
Inserting the previous inequality in the definition of $x_0$ and using $n_0\leq x_0+1$ allows us to conclude
\[
n_0 \leq \frac{A + B \log\big(2(B^2+AC) /C^2\big)}{C}+1\,.
\]
\end{proof}
\section{Experiments}
\label{app:experiments}

In this section we conduct preliminary experiments for the cumulative regret and best-arm identification setting.

\paragraph{Cumulative regret} For the cumulative regret we compare \OurAlgorithm (with $\gamma = 0.5$) with the QRM1 algorithm by  \citep{chaudhuri2018quantile} and SR algorithm by \citep{zhu2020regret}. We arbitrarily\footnote{Which is not very important, since we evaluate the algorithms from a problem-dependent point of view} choose the following reservoir: the arms are distributed according to a Bernoulli distribution with possible means $[0.5,\, 0.8]$ sampled with probabilities $[0.8,\, 0.2]$. We remark that the SR algorithm and \OurAlgorithm are very similar, they both sample approximately $\log(T)/\pstar$ arms and run a regret minimizer algorithm on this set of arms. The only difference is that the SR algorithm relies on the MOSS algorithm. Whereas the QRM1 algorithm proceeds by progressively adding new arms. In particular this algorithm is anytime. In Figure~\ref{fig:regret} we compare the cumulative regret of the different algorithms for a fixed horizon $T=20000$. We observe that \OurAlgorithm behaves similarly to SR and that QRM1 performs slightly worst (maybe because of the adaptation to $T$). We also check that all algorithms have a regret that is logarithmic with the horizon as expected. To this aim, in Figure~\ref{fig:regret_T}, we plot the cumulative regret (for the same reservoir) for all horizons $T\in\{100,200,\ldots,10000\}$.

\begin{figure}[!t]
\centering

\includegraphics[width=0.8\textwidth]{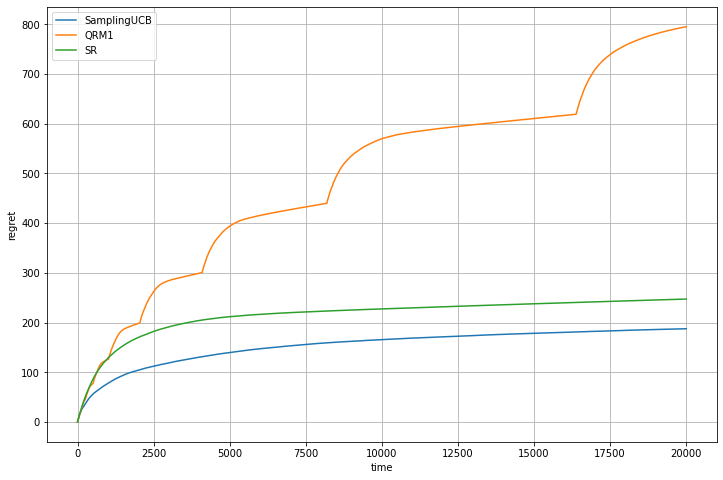}
\caption{Cumulative regret in function of the time estimated by $100$ Monte-Carlo simulations.}
\label{fig:regret}
\end{figure}

\begin{figure}[!t]
\centering

\includegraphics[width=0.8\textwidth]{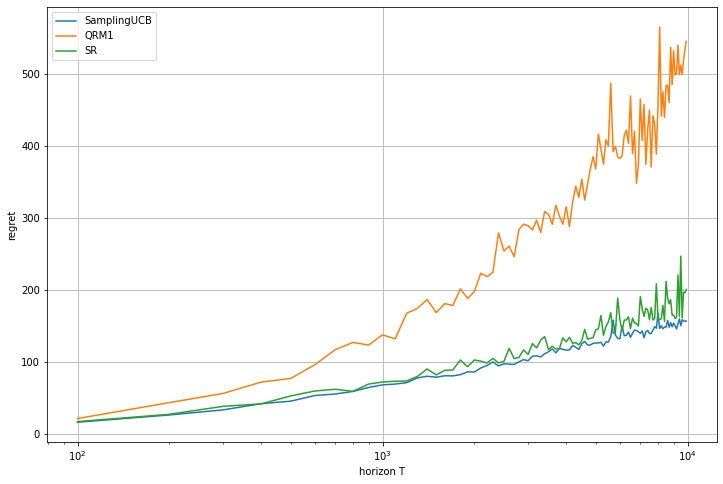}
\caption{Cumulative regret in function of the horizon $T\in\{100,200,\ldots,10000\}$ estimated by $100$ Monte-Carlo simulations.}
\label{fig:regret_T}
\end{figure}

\paragraph{Best-arm identification} For best arm identification we compare our algorithm with the BUCB algorithm by \citep{samuels2020complexity}. In Figure~\ref{fig:bai} we compare the performance of the algorithms across varying $\Delta$ for a fixed $T=1000$. That is, we consider reservoirs of the form $[0.2,\Delta,1]$ for $\Delta \in (0.01 \times i)_{i\in[79]}$ with probabilities $[0.29,0.69,0.02]$. The BUCB algorithm presents an issue as it is designed for the fixed confidence regime the algorithm takes $\delta$ as a parameter. We set $\delta$ equal to an arbitrarily low constant. The BUCB algorithm works by opening successively large brackets of arms, however as they do not provide results in high probability, only in expectation, they can draw significantly less arms from the reservoir. The performance of \OurAlgorithmSimple seems favourable compared to BUCB, however, one may be able to improve the performance of BUCB with parameter tuning. 

\begin{figure}[!t]
\centering

\includegraphics[width=0.8\textwidth]{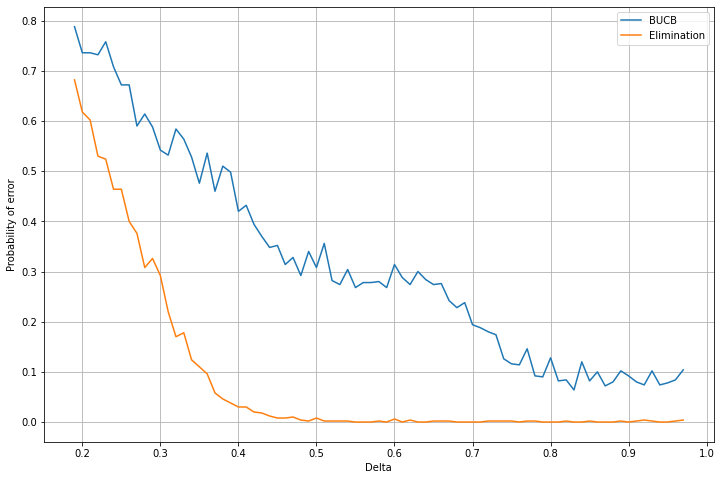}
\caption{Probability of error for best arm identification across varying $
\Delta$ using  $500$ Monte-Carlo simulations.}
\label{fig:bai}
\end{figure}

\end{document}